%% file: main.tex
\documentclass{article}

\usepackage[preprint]{corl_2026}

\usepackage{graphicx}
\usepackage{subcaption}
\usepackage{booktabs}

\usepackage{algorithm}
\usepackage{algpseudocode}
\usepackage{float}
\usepackage{amsmath}
\usepackage{amssymb}
\usepackage{mathtools}
\usepackage{amsthm}
\usepackage{enumitem}

\theoremstyle{plain}
\newtheorem{theorem}{Theorem}[section]

\theoremstyle{definition}

\theoremstyle{remark}

\title{ALOE: Action-Level Off-Policy Evaluation for Vision-Language-Action Model Post-Training}

\author{%
\textbf{Rushuai Yang}\textsuperscript{1,2,*} \quad
\textbf{Hecheng Wang}\textsuperscript{1,3,*} \quad
\textbf{Zhichao Wu}\textsuperscript{1,4,*} \quad
\textbf{Chiming Liu}\textsuperscript{1,*,\textdagger} \quad
\\
\textbf{Xiaohan Yan}\textsuperscript{1} \quad
\textbf{Xuan Du}\textsuperscript{1} \quad
\textbf{Shuoyu Yue}\textsuperscript{1} \quad
\textbf{Chuheng Zhang}\textsuperscript{5} \quad
\textbf{Yunlong Wang}\textsuperscript{1} \quad
\\
\textbf{Yongcheng Liu}\textsuperscript{1} \quad
\textbf{Lizhe Qi}\textsuperscript{3} \quad
\textbf{Yi Chen}\textsuperscript{2} \quad
\textbf{Wei Shan}\textsuperscript{1} \quad
\textbf{Maoqing Yao}\textsuperscript{1}
\\
\textsuperscript{1}AgiBot \quad
\textsuperscript{2}The Hong Kong University of Science and Technology \quad
\textsuperscript{3}Fudan University \quad
\\
\textsuperscript{4}Nanjing University \quad
\textsuperscript{5}Independent Researcher \quad
\textsuperscript{*}Equal Contribution \quad
\textsuperscript{\textdagger}Corresponding Author
}

\begin{document}
\maketitle

% \vspace{0.3\baselineskip}
\input{sec/abstract}
\input{sec/introduction}
\input{sec/related_work}

\input{sec/preliminaries}
\input{sec/method}

\input{sec/experiments}

\input{sec/conclusion}

\bibliography{references}

\clearpage
% \appendix
\input{sec/appendix}

\end{document}

%% file: sec/abstract.tex
\begin{abstract}
    We study how to improve large foundation vision-language-action (VLA) systems through human-in-the-loop reinforcement learning (RL) in real-world environments. A key challenge is learning reliable value functions from heterogeneous real-world experience, as value estimation provides the primary learning signal for VLA training. In practice, replay buffers contain trajectories collected from historical policies, online rollouts, demonstrations, and intermittent human interventions. Because replay buffers mix trajectories generated by different behaviors, the observed returns can be mismatched with the quality of the current policy. Prior VLA post-training methods often rely on progress-style value signals, which reflect the average quality of historical behaviors, leading to mismatched learning signals for the current policy. In this paper, we propose ALOE, an off-policy evaluation framework whose value function directly evaluates current-policy behavior for each iteration. Specifically, ALOE combines chunked temporal-difference bootstrapping and conservative value aggregation to perform stable current-policy evaluation, then uses these estimates for advantage-weighted policy improvement. This design improves credit assignment to critical action chunks under sparse rewards and supports stable policy improvement. We evaluate ALOE on four real-world manipulation tasks encompassing long-horizon and high-precision scenarios: smartphone packing, laundry folding, multi-object sorting, and phone assembly. Across all tasks, ALOE outperforms other VLA post-training methods, highlighting the benefit of off-policy value estimates for real-world VLA post-training. Videos are available at our project website \href{https://rooshy-yang.github.io/aloe/\#}{https://rooshy-yang.github.io/aloe}.
    
    % We study how to perform iterative reinforcement learning (RL) post-training for large foundation vision-language-action (VLA) systems in real-world settings. This iterative approach provides a highly flexible and practical paradigm for real-world implementation. However, a major challenge arises in reliably improving the current policy with offline-collected data, particularly when it consists of diverse sources like autonomous rollouts from historical policies and human interventions. Guiding the actor through such mixed data requires accurate evaluation, making stable and reliable critic learning the core problem. Prior works conservatively evaluates the mean performance from mixture datasets, which can bias the actor and limit learning effectiveness. In this paper, we propose ALOE, an action-chunk level off-policy evaluation framework that learns a chunk-based Q-function. This design enables the critic to accurately identify critical action chunks in sparse reward scenarios, providing more reliable guidance for policy improvement.
    % We evaluate our method on four real-world manipulation tasks, including both high-precision and long-horizon tasks. Empirical results show that ALOE can reliably assess the contribution of individual action chunks, achieving faster improvements in task success rates compared to prior baselines. Videos and additional materials are available at our project website \href{https://rooshy-yang.github.io/aloe/\#}{https://rooshy-yang.github.io/aloe}.
    \end{abstract}
    \vspace{-1\baselineskip}
    

%% file: sec/introduction.tex
\section{Introduction}
\label{sec:introduction}
Vision-language-action (VLA) models have emerged as a promising paradigm for robotic manipulation, enabling policies to condition actions on both real-world visual observations and natural language instructions~\cite{agibotworldcontributors2025agibotworldcolosseolargescale,pi0,pi05,yang2025magma,zhao2025cot,lin2024vila,szot2025multimodal,li2024manipllm}. Recent progress has shown that combining imitation learning with reinforcement learning (RL) can substantially improve task performance beyond supervised learning~\cite{recap,rldg,janny2025reasoning,hoeller2024anymal}. A key factor behind these improvements is the value function, which evaluates behavior quality and provides optimization signals for VLA policies. However, accurately estimating value functions in real-world VLA systems remains particularly challenging~\cite{recap,Ma2025GVL}. Because real-world interaction is costly, VLA RL must reuse heterogeneous replay data, turning policy evaluation into an off-policy problem. Naively applying off-policy value learning can suffer from error accumulation under distributional uncertainty~\cite{kumar2019stabilizing,deadlytriad}. To preserve stability for high-capacity VLA systems, prior real-world methods often use conservative on-policy or state-value style evaluation~\cite{recap,awr}. These approaches can be robust, but the learned value function often reflects an implicit mixture behavior distribution rather than directly evaluating the current policy's individual actions. While this conservative design improves training stability, it relies on a strong assumption: the behavior distribution of the evolving policy remains close to the mixture distribution represented in replay. This assumption is fragile when human interventions, policy updates, and environment variability continuously shift the data distribution. If the current VLA policy learns new recovery behaviors that were absent from historical data, a coarse state-value or progress-style critic may underestimate these actions and provide misleading optimization signals. This induces a trade-off between stability and learning efficiency in real-world VLA training, motivating the following question:
\begin{center}
    \textit{How to learn a reliable value function with mixed real-world replay for flow-matching VLA?}
\end{center}
In this work, we introduce ALOE, an action-level off-policy evaluation framework for real-world post-training of large flow-based VLA policies. ALOE learns $Q(s,a)$ with temporal-difference bootstrapping so value information can be stitched across fragmented replay. It uses Q-chunking to accelerate sparse-reward credit propagation over long horizons, and pessimistic ensemble aggregation to reduce overestimation on uncertain actions. The learned critic then guides an advantage-weighted update for the flow-based actor. This design allows VLA policies to use fine-grained action-level value information while preserving the stability required for real-world robot training. We summarize our contributions as follows:
\begin{itemize}[leftmargin=0.8em]
\item We propose an action-level off-policy evaluation framework for real-world VLA post-training under heterogeneous replay from online rollouts, demonstrations, and human interventions.
\item We combine Q-chunked TD learning and pessimistic critic aggregation to enable stable long-horizon action-value learning for large flow-based VLA policies.
\item We evaluate ALOE on four real-world robotic manipulation tasks and provide ablations and critic diagnostics showing that the main gains come from action-level off-policy critic design.
\end{itemize}

\begin{figure}[t]
    \centering
    \includegraphics[width=\linewidth]{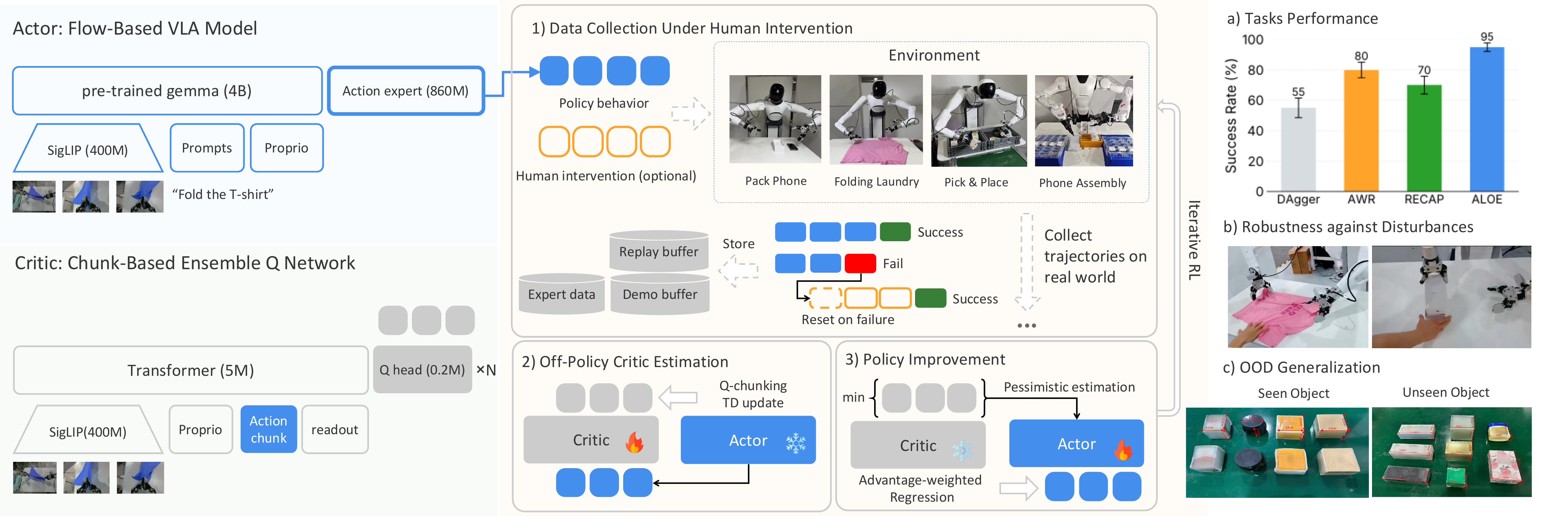}
    \caption{\textbf{Overview of ALOE for real-world VLA post-training.} \textit{Left}: ALOE pairs a flow-based VLA actor with a chunk-based ensemble critic that scores action sequences. \textit{Middle}: Training iterates between human-in-the-loop data collection, off-policy critic learning, and advantage-weighted policy improvement. \textit{Right}: Evaluation covers task success, robustness to disturbances, and zero-shot generalization to unseen objects.}
    \label{fig:teaser}
\end{figure}
\vspace{-1\baselineskip}

%% file: sec/related_work.tex
\section{Related Work}

\subsection{Foundations of Policy Evaluation}
Policy evaluation estimates the expected long-term outcome of actions or policies and is commonly categorized along two axes: on-policy vs.\ off-policy estimation, and Monte Carlo (MC) targets vs.\ temporal-difference (TD) bootstrapping targets~\cite{sutton1998reinforcement}.
On-policy methods estimate returns under the data-collecting policy, which improves stability but limits data reuse~\cite{dppo,ppo,Ren2025DPPO}.
Off-policy methods estimate the value of actions or target policies from data collected by other behavior policies, improving sample efficiency but introducing distributional uncertainty~\cite{levine2016end,kalashnikov2018qt,sacflow,reinflow}.
MC targets rely on observed trajectory returns and are sensitive to sparse rewards and trajectory fragmentation, while TD targets bootstrap locally and can propagate value across disconnected transitions~\cite{sutton1998reinforcement,TD,kumar2020cql,dabney2018qrdqn}.
ALOE builds on these classical tools, but the contribution is their controlled use for action-level evaluation of a large flow-based VLA under complex real-world manipulation tasks than the individual ingredients themselves.

\subsection{Reinforcement Learning for Large VLA Policies}
Recent work has explored scaling RL to large VLA models on real robotic systems~\cite{Tan2025InteractivePostTraining,Lu2025VLARL,Liu2025WhatCanRLBringVLA,Chen2025piRL,Li2025SimpleVLA_RL}.
On-policy methods such as PPO, DPPO, and REINFORCE prioritize stability by evaluating the current rollout distribution~\cite{williams1992simple,Ren2025DPPO,ppo,dppo}, but this limits reuse of costly real-world data.
Another line uses trajectory-level preference, progress, or state-value modeling to reweight imitation-style updates~\cite{recap,Zhai2025VLAC,Ghasemipour2025SelfImprovingEFM,Ma2023LIV,Ma2025GVL}.
These methods, including AWR and $\pi_{0.6}$-style training (RECAP)~\cite{awr,recap}, are closely related to our actor update, but their value signal is primarily state-level or progress-style rather than a current-policy action-value critic. Action-sequence and action-value methods for VLA have also been studied in offline or restricted settings~\cite{Huang2025CORFT,Frans2025DiffusionGuidance,kuba2023advantage}.
CO-RFT and DEAS-style methods show that chunked or action-sequence value learning is useful for VLA fine-tuning, but they focus mainly on offline data with no online rollout or different training regimes.
Residual and data-generation approaches such as RoboFuME and PLD-style self-improvement use smaller Gaussian or residual policies to guide data collection or policy refinement~\cite{Yang2024RoboFuME,xiao2025selfimprovingvisionlanguageactionmodelsdata}.
Simulation-oriented methods such as SimpleVLA-RL scale RL with simulator support~\cite{Li2025SimpleVLA_RL}.
In contrast, ALOE directly improves the main large flow-based VLA policy in a real-world online loop with human interventions, safety-driven truncations, and mixed-policy replay.

%% file: sec/preliminaries.tex
\begin{figure*}[t]
    \centering
    \includegraphics[width=\linewidth]{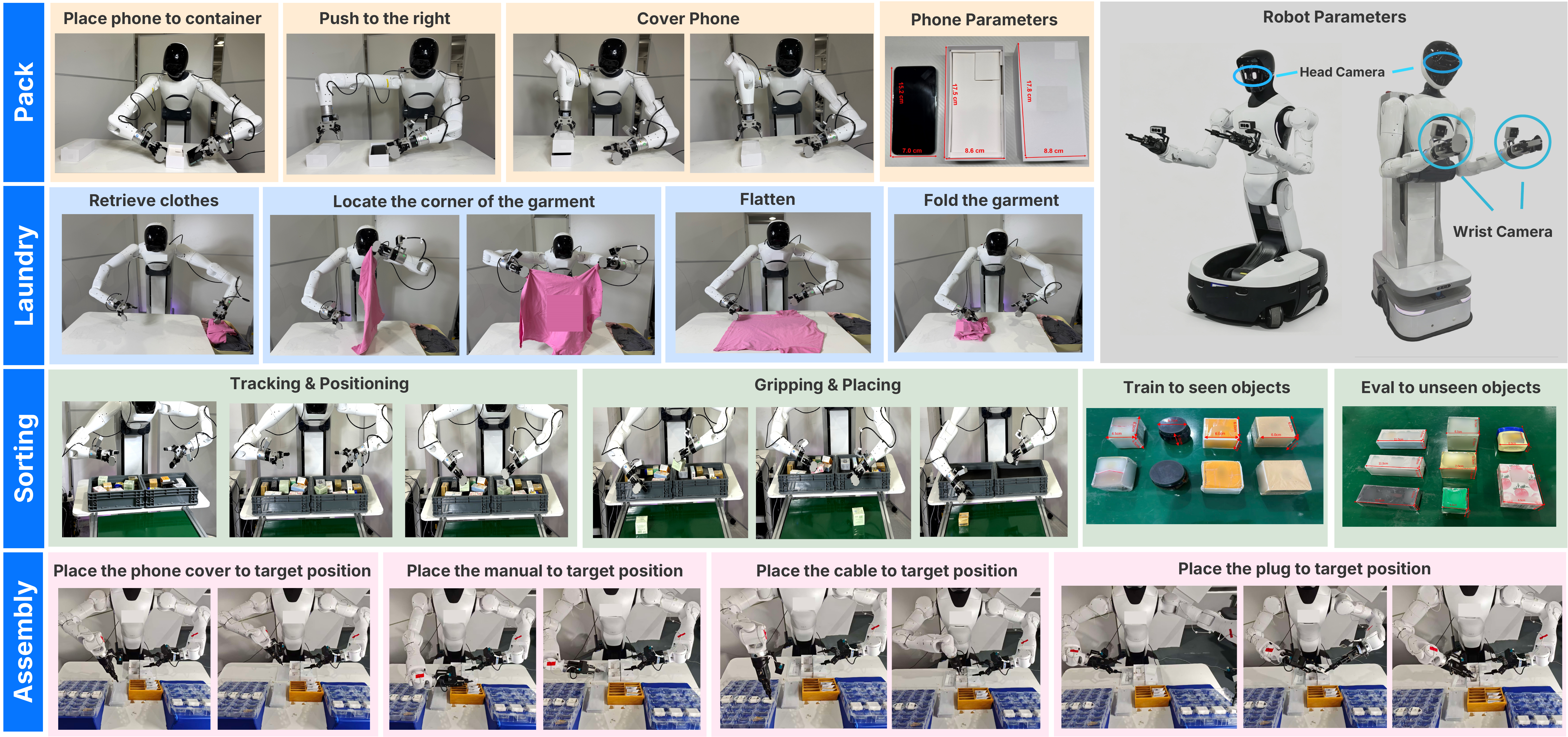}
    \caption{\textbf{Illustrations of our real-world manipulation tasks (phone packing, laundry folding, multi-object sorting, and high-precision phone assembly) and the two robotic setups used in the experiments.} From left to right, each task visualization depicts the sequence of subtasks that must be completed sequentially.}
    \label{fig:exp_plot}
\end{figure*}
\section{Preliminaries}
\label{sec:preliminaries}

\subsection{Real-World VLA Post-Training Pipeline}

We study post-training of an instruction-conditioned VLA policy $\pi_\theta(a_t\mid s_t, \ell)$, where $s_t$ is the robot observation, $\ell$ is the language instruction, and $a_t$ is the action.
The policy is rolled out on the task. Each episode is labeled with a task outcome that determines the reward, and human interventions may be provided to correct early-iteration failures.
Unlike simulated RL, real-world VLA post-training cannot rely on large-scale on-policy rollouts due to slow interaction, costly resets, and limited parallelism.
Training data is therefore accumulated in a replay buffer $\mathcal{D}$ containing transitions $(s_t,a_t,r_t,s_{t+1},d_t,\ell)$ from previous policies, 
demonstrations, autonomous rollouts, and human interventions. A value function is then trained on the accumulated replay data and used to guide policy improvement.
\subsection{Value Estimation: Data Policy vs. Current Policy}
\label{sec:value_estimation}
For high-dimensional visual inputs, using episode return as the value target often leads to high variance estimates, especially in real-world RL.
Prior VLA post-training methods therefore commonly learn a value or progress model as the policy-improvement signal.
Real-world VLA data is collected through repeated rollouts of an evolving policy with safety-driven truncations and human interventions.
The replay buffer $\mathcal{D}$ therefore mixes online rollouts, historical policies, demonstrations, and human corrections, including fragments in which different segments of one task attempt may come from different policies or teleoperation.
Let $\pi_{\mathcal{D}}$ denote the implicit data policy induced by this buffer, i.e., the conditional action distribution represented by replay actions rather than a single executable checkpoint.
For example, an MC estimator learns a state value by fitting the observed return:
\begin{equation}
V^{\pi_{\mathcal{D}}}
\in \operatorname*{arg\,min}_{V}
\mathbb{E}_{\mathcal{D}}
\left[
\left(V(s_t, \ell)-\sum_{k=0}^{T-t}\gamma^k r_{t+k}\right)^2
\right],
\end{equation}
while a SARSA-style critic bootstraps from the next replay action,
\begin{equation}
Q^{\pi_{\mathcal{D}}}
\in \operatorname*{arg\,min}_{Q}
\mathbb{E}_{\mathcal{D}}
\left[
\left(Q(s_t,a_t, \ell)-\left(r_t + \gamma(1-d_t)Q(s_{t+1},a_{t+1},\ell)\right)\right)^2
\right].
\end{equation}
Thus, $\pi_{\mathcal{D}}$ may combine old checkpoints, human teleoperation segments, and corrected recovery trajectories that do not correspond to one executable VLA policy that can be rolled out.
Values learned from replay returns or replay next-actions may therefore describe progress under $\pi_{\mathcal{D}}$, not the current policy being optimized. Using such values for policy improvement can therefore provide misaligned learning signals, especially in real-world settings. This mismatch can overestimate actions that are not reliably executable by the current policy and weaken subsequent policy improvement. To evaluate the current policy from mixed replay, we can use a TD-style off-policy objective.
It still uses the observed transition, but bootstraps from the next action sampled from the current policy:
\begin{equation}
\label{eq:q_def}
Q^\pi
\in \operatorname*{arg\,min}_{Q}
\mathbb{E}_{\mathcal{D}}
\left[
\left(Q(s_t,a_t, \ell)-\left(r_t+\gamma(1-d_t)
\mathbb{E}_{a'\sim\pi(\cdot \mid s_{t+1}, \ell)}
\left[Q(s_{t+1},a',\ell)\right]\right)\right)^2
\right].
\end{equation}
Unlike SARSA, which bootstraps from the next replay action $a_{t+1}$, this target evaluates actions followed by the current policy. However, off-policy TD learning combines the ingredients of the deadly triad: bootstrapping, off-policy data, and function approximation, which can destabilize learning~\cite{deadlytriad}.
This issue is especially serious in real-world VLA replay, where sparse rewards slow credit propagation and many current-policy actions are weakly covered by the replay buffer.
ALOE adopts this view and stabilizes it with Q-chunking and pessimistic aggregation in Sec.~\ref{sec:method_critic}.

%% file: sec/method.tex
\section{Method}
\label{sec:method}

As shown in Fig.~\ref{fig:teaser}, ALOE targets real-world VLA post-training with fragmented, policy-mixed replay from human-in-the-loop data collection. It learns a current-policy off-policy critic with Q-chunking and pessimistic aggregation (Sec.~\ref{sec:method_critic}), then uses the critic for advantage-weighted policy improvement of a flow-based VLA actor (Sec.~\ref{sec:method_actor}).

\subsection{Stable Off-Policy Critic Learning}
\label{sec:method_critic}

\begin{figure*}[t]
    \centering
    \includegraphics[width=\linewidth]{figs/Q_value2_cropped.png}
    \caption{\textbf{ALOE provides fine-grained action-level credit assignment.} \textbf{Left:} on the Phone Packing task, ALOE identifies recovery behaviors, where the policy transitions from failed attempts (red) to successful execution (green). ALOE also captures subtle action adjustments when trying to insert the phone case. \textbf{Right:} on the Laundry Folding task, ALOE detects negative progress and assigns lower values to failure-prone behaviors such as unsuccessful flattening or grasping.}
    \label{fig:q_value_visualization}
\end{figure*}
\vspace{-1mm}
\begin{figure*}[b]
    \centering
    \includegraphics[width=\linewidth]{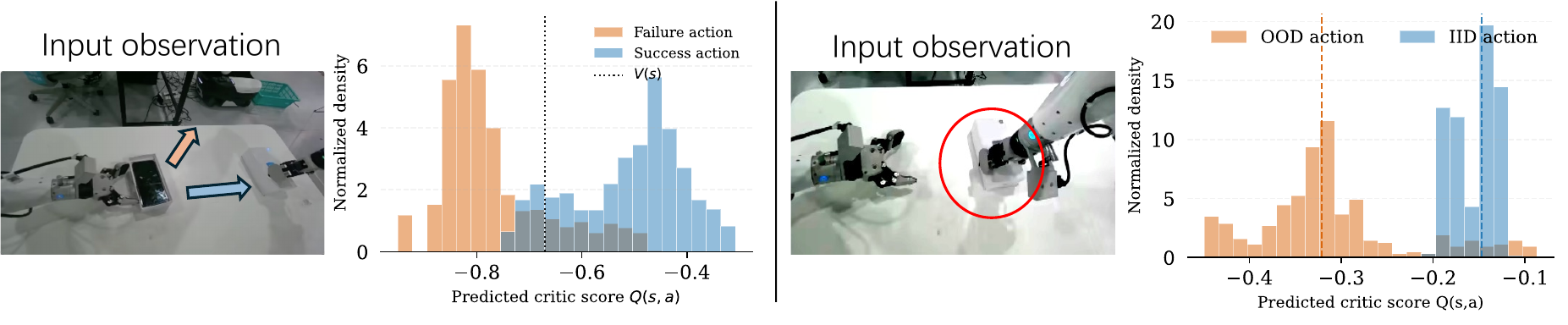}\caption{\textbf{Critic diagnostics.} We visualize the predicted Q-score distributions to examine whether the critic can (i) distinguish good and bad actions under similar states and (ii) behave conservatively on OOD actions. \textbf{Left:} actions leading to successful trajectories receive consistently higher scores than those leading to failures. \textbf{Right:} unseen OOD actions (e.g., random shaking behaviors by purpose) are assigned lower scores than actions that lead to precise insertion. These results demonstrate both fine-grained action ranking and conservative value estimation under distribution shift.}
    \label{fig:q_value_all}
\end{figure*}

We first describe the policy evaluation component of our method.
ALOE instantiates the current-policy off-policy evaluation objective from Sec.~\ref{sec:value_estimation}.
Rather than using replay returns or replay next-actions to evaluate the implicit data policy $\pi_{\mathcal{D}}$, ALOE bootstraps through actions sampled from the current actor.
This allows value information to propagate across fragmented transitions while keeping the critic aligned with the policy being optimized.
However, standard one-step TD propagates value slowly under sparse rewards.
To accelerate credit propagation, we adopt Q-chunking~\cite{q-chunking}, which backs up value over multi-step action chunks while preserving the current-policy bootstrap.
Specifically, we define a chunked action sequence
\(
\mathbf{a}_{t:t+h} \triangleq (a_t, a_{t+1}, \dots, a_{t+h-1}),\) and learn a chunked action-value function
$Q^\pi(s_t, \mathbf{a}_{t:t+h}, \ell)$
that evaluates the return of executing this exact action sequence from state $s_t$ under instruction $\ell$. For a transition segment
$(s_t, \mathbf{a}_{t:t+h}, r_{t:t+h-1}, s_{t+h}, d_{t+h-1}, \ell) \sim \mathcal{D}$,
the Q-chunking Bellman target is defined as
\begin{equation}
y_t^{(h)} =\sum_{k=0}^{h-1} \gamma^k r_{t+k}+\gamma^h(1-d_{t+h-1})
\mathbb{E}_{\mathbf{a}' \sim \pi_\theta(\cdot \mid s_{t+h}, \ell)} \big[Q^\pi_{\phi^-}(s_{t+h}, \mathbf{a}', \ell)\big],
\label{eq:q_chunk_target}
\end{equation}
where $\mathbf{a}' = \mathbf{a}_{t+h:t+2h}$ denotes the next action chunk sampled from the current policy.
Unlike standard $n$-step TD, this backup uses the executed action chunk for intermediate rewards, accelerating value propagation while preserving the current-policy bootstrap~\cite{q-chunking}. In regions of the state-action space with insufficient observations, we find that traditional TD-based methods may extrapolate aggressively and produce overestimated action values. This effect is amplified in safety-critical manipulation tasks, where exploration is constrained and data diversity is inherently limited. To mitigate the overestimation arising from uncertainty, we employ a pessimistic ensemble of $K$ action-value functions $\{Q^\pi_{\phi_i}\}_{i=1}^K$. Each critic is trained using the same chunked TD target in
Eq.~\eqref{eq:q_chunk_target},
by minimizing the objective:
\vspace{-2mm}
\begin{equation}
\mathcal{L}_{\mathrm{critic}}
=
\frac{1}{K}\sum_{i=1}^K
\mathbb{E}_{\mathcal{D}}
\left[
\big(
Q^\pi_{\phi_i}(s_t, \mathbf{a}_{t:t+h}, \ell) - y^{(h)}_{t}
\big)^2
\right],
\label{eq:critic_loss}
\end{equation}
where $y^{(h)}_{t}$ denotes the chunked TD target.
At policy update time, we form a conservative estimate as $
Q^\pi_{\mathrm{pess}}(s_t, \mathbf{a}_{t:t+h}, \ell)
=\min_{i} Q^\pi_{\phi_i}(s_t, \mathbf{a}_{t:t+h}, \ell)$,
which approximates a lower confidence bound on the true action value \cite{whysopessimistic, lockwood2022review,tamar2015cvar,chow2017risk}. This aggregation evaluates the current policy while providing more sensitive identification of uncertainty, improving robustness under real-world deployment.

% \label{sec:experiments}
\subsection{Advantage-Weighted Policy Improvement for VLA Policies}
\label{sec:method_actor}

We now describe how the learned action-value function is used to update the VLA policy. Although the critic is learned off-policy, our value evaluation integrates naturally with the constrained policy update strategy \cite{awr,awac}, treating the critic as a relative preference signal. This preserves constraint updates and ensures stable policy improvement across online iterations.
Given a transition, we evaluate the quality of the data action relative to the current policy $\pi$ by the advantage $
A^\pi(s_t, \mathbf{a}_{t:t+h}, \ell) = Q^\pi_{\mathrm{pess}}(s_t, \mathbf{a}_{t:t+h}, \ell) - \mathrm{sg}(V^\pi(s_t, \ell))$ with $V^\pi(s_t, \ell) = \mathbb{E}_{\mathbf{a}_{t:t+h}' \sim \pi_\theta(\cdot \mid s_t, \ell)}[Q^\pi_{\mathrm{pess}}(s_t, \mathbf{a}_{t:t+h}',\ell)]$
is estimated by sampling action chunks from the current policy, and $\mathrm{sg}$ prevents actor gradients from flowing through the critic baseline.
This advantage compares replay actions against the current policy at the same state and instruction. To emphasize relative action quality while preserving stability, we transform it into a non-negative weight using a clipped exponential,
\vspace{-0.5mm}
\begin{equation}
w(s_t, \mathbf{a}_{t:t+h}, \ell)
=
\exp\!\left(
\mathrm{clip}\!\left(
A^\pi(s_t, \mathbf{a}_{t:t+h}, \ell) / \beta,
-\epsilon_{\mathrm{clip}},
\epsilon_{\mathrm{clip}}
\right)
\right),
\label{eq:aw_weight}
\end{equation}
where $\beta$ controls the sharpness of the weighting
and $\epsilon_{\mathrm{clip}}$ limits the maximum influence of any single transition,
analogous to trust-region clipping \cite{schulman2017proximal}.
The policy is then updated by maximizing the weighted log-likelihood
over the entire replay buffer, $\mathcal{L}_{\mathrm{actor}}=
\mathbb{E}_{\mathcal{D}}
\left[
w(s_t, \mathbf{a}_{t:t+h}, \ell)\,
\log \pi_\theta(\mathbf{a}_{t:t+h} \mid s_t, \ell)
\right].$
This actor update is data-support constrained: it only increases the likelihood of actions observed in replay, with larger weights assigned to actions that the current critic ranks above the policy's average behavior.
The clipped exponential weight further limits the influence of any single transition, which stabilizes optimization for large flow-based VLA models.

\vspace{-1mm}
\begin{figure*}[h]
    \centering
    \includegraphics[width=\linewidth]{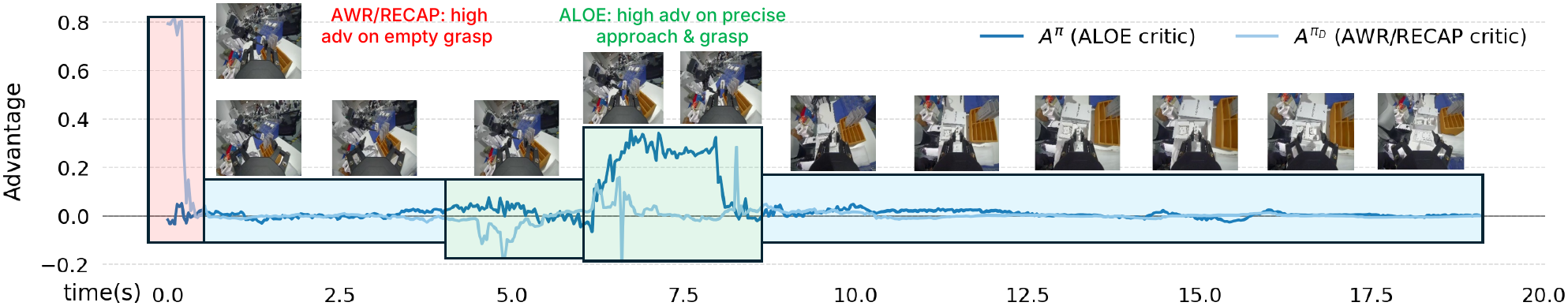}
    \caption{\textbf{Advantage comparison on Phone Assembly task.} ALOE assigns high advantage to task-relevant approach-and-grasp actions (green region), whereas AWR/RECAP can misattribute advantage to meaningless empty-grasp states (red region).}
    \label{fig:adv_compare}
\end{figure*}
\vspace{-4mm}

\subsection{Implementation Details}
\label{sec:detailed}
We instantiate ALOE on the $\pi_{0.5}$ flow-matching VLA backbone~\cite{pi05} and fine-tune the actor end-to-end during RL.
The critic uses the same visual observations and proprioception as the actor, with a Transformer-based multimodal encoder and an ensemble of Q readout heads for pessimistic aggregation.
Additional architecture, reward, and optimization details are provided in Appendix~\ref{app:training}. Our actor parameterizes continuous action chunks via flow matching, so the continuous log-likelihood term
$\log \pi_\theta(\mathbf{a}_{t:t+h}\mid s_t, \ell)$ in $\mathcal{L}_{\mathrm{actor}}$, is not tractable to evaluate exactly.
Following prior flow/diffusion-style VLA training practice~\cite{recap}, we optimize a tractable surrogate by treating the negative flow-matching objective as a proxy for the continuous log-likelihood, i.e.,
\vspace{-1mm}
\begin{align}
\label{eq:flow_actor}
\mathcal{L}_{\mathrm{aloe}}
&=\mathbb{E}_{\mathcal{D}}
\Big[
w(s_t, \mathbf{a}_{t:t+h}, \ell)
\bigl\lVert
\boldsymbol{\epsilon}-\mathbf{a}_{t:t+h}
- f_\theta\!\bigl(\tilde{\mathbf{a}}_{t:t+h}, s_t, \ell\bigr)
\bigr\rVert_2^2
\Big]
\end{align}
where $(s_t, \mathbf{a}_{t:t+h}, \ell)\sim\mathcal{D},
\eta\sim\mathcal{U}[0,1], \epsilon\sim\mathcal{N}(\mathbf{0},\mathbf{I}),$ the noised action is defined as
$\tilde{\mathbf{a}}_{t:t+h}=\eta\mathbf{a}_{t:t+h}+(1-\eta)\boldsymbol{\epsilon}$. Intuitively, maximizing the weighted policy log likelihood corresponds to minimizing weighted flow-matching loss for the action chunks under the same inputs.

%% file: sec/experiments.tex
\section{Experiments}
We evaluate ALOE on real-world manipulation tasks designed to stress long-horizon reasoning, precise action selection, robustness, and learning from human-in-the-loop replay.
The central experimental questions are: (i) whether action-level off-policy evaluation improves flow-based VLA post-training, (ii) whether the gains come primarily from the critic-side design, and (iii) whether the learned critic reliably scores current-policy and out-of-distribution actions.
Fig.~\ref{fig:q_value_visualization} provides a qualitative example of the learned critic along representative target-policy trajectories.

\subsection{Experimental Setup}

\paragraph{Tasks and Baselines.}
We evaluate ALOE on four representative real-world robotic manipulation tasks shown in Fig.~\ref{fig:exp_plot}: Pack Smart Phone, Folding Laundry, Product Sorting, and Phone Assembly. Our main comparison uses baselines available across all tasks under the same $\pi_{0.5}$ flow-based VLA backbone: DAgger and AWR. DAgger does behavior cloning with success-only data from same replay buffer without explicit value estimation. AWR~\cite{awr} uses MC state-value estimation $V^{\pi_D}(s)$ from replay and extracts the policy the same as ALOE, serving as the state-value baseline. We additionally report RECAP~\cite{recap} on the assembly task at supported robot platform. RECAP also uses MC state-value estimation $V^{\pi_D}(s)$ but using classifier-free-guidance (CFG) for actor extraction. Detailed task parameters, reward definitions, and data collection protocols are provided in Appendix.

\subsection{Overall Policy Performance}

\begin{figure*}[t]
    \centering
    \includegraphics[width=\linewidth]{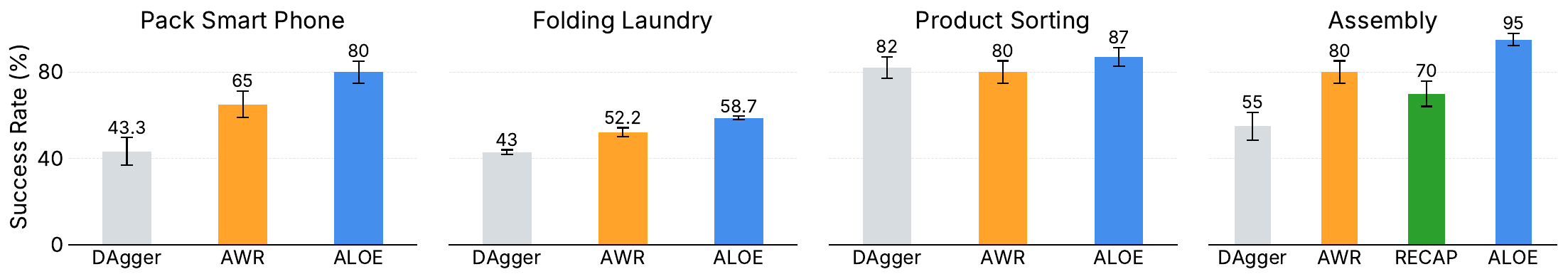}
\caption{\textbf{Average success rates on four manipulation tasks.} Each final policy is evaluated over 60 trials per task; ALOE achieves the highest final success rate across all tasks.}
    \label{fig:all_results}
\end{figure*}

\begin{figure*}[t]
    \centering
    \includegraphics[width=\linewidth]{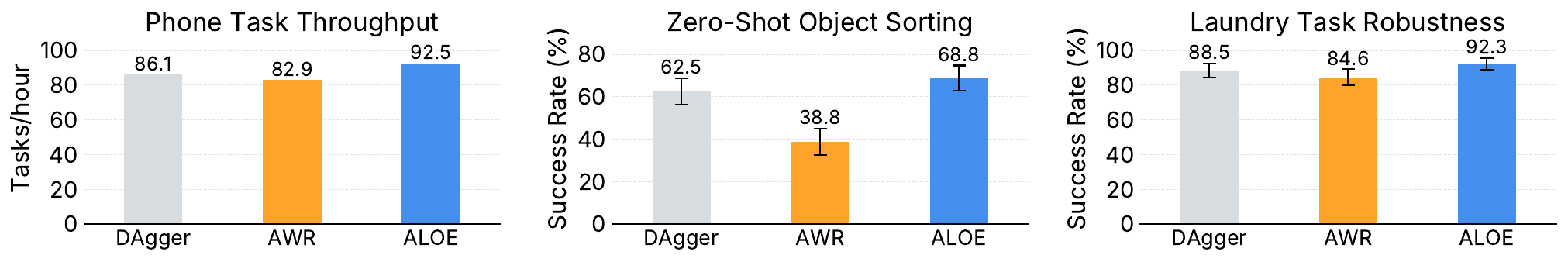}
\caption{\textbf{Efficiency, zero-shot generalization, and robustness.}
Left: phone-packing throughput is measured as the number of successful task completions per hour.
Middle: zero-shot generalization is evaluated on product-sorting objects with sizes and shapes unseen during training. Right: robustness is evaluated by applying external disturbances during laundry-folding execution and measuring recovery success.}
    \label{fig:throughput_generalization}
\end{figure*}

\begin{figure*}[t]
    \centering
    \includegraphics[width=\linewidth]{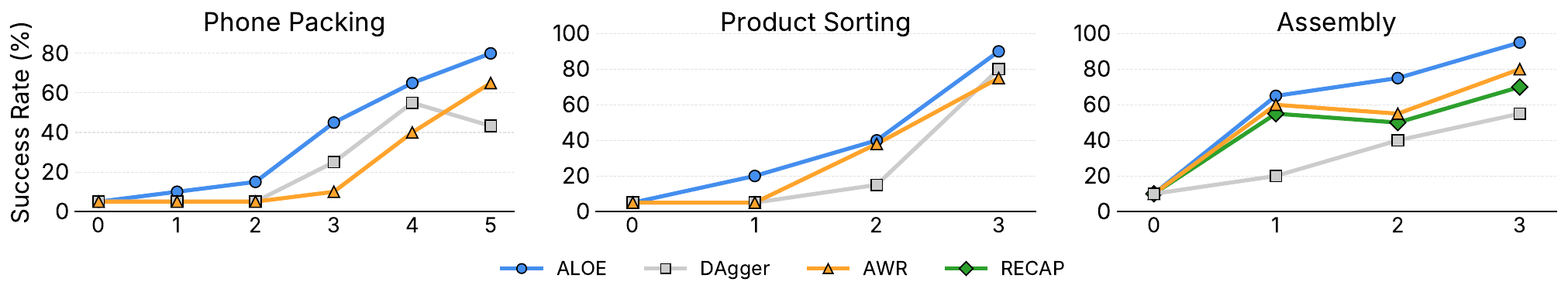}
    \caption{\textbf{Policy improvement across RL iterations.} $i$ on X axis refers to the policy iteration steps and $0$ is the BC warm-up using small human demonstration. Each iteration collects a fixed amount of online data.}
    \label{fig:policy_improvement}
\end{figure*}
\vspace{-1mm}

\begin{figure*}[tp]
\centering
\includegraphics[width=\linewidth]{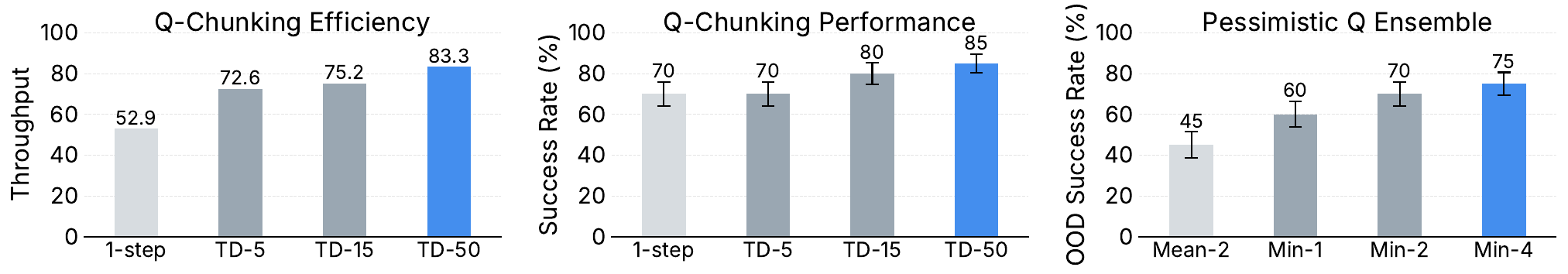}
\caption{\textbf{Critic component analysis.} Left and middle: TD-$X$ denotes the chunk size used during TD learning. Larger chunk sizes improve both task throughput and success rate on the Phone Packing task. Right: Mean/Min-$X$ denotes the ensemble size and whether minimum-value aggregation is used during policy extraction. Pessimistic aggregation achieves higher OOD success rates on unseen objects in the Product Sorting task, indicating improved robustness under distribution shift.}
\label{fig:critic_components}
\end{figure*}
\vspace{-1mm}
We first evaluate whether ALOE improves the final policy. Fig.~\ref{fig:all_results} shows that ALOE achieves the best final success rate on all four tasks. Compared with AWR, ALOE yields the largest gain on the high-precision Phone task, indicating $Q^\pi$ performs better than $V^{\pi_D}$. Fig.~\ref{fig:policy_improvement} further shows that ALOE continues to improve across RL iterations. On Assembly, where we additionally include RECAP, ALOE also outperforms this baseline that replaces AWR policy extraction with CFG.
Fig.~\ref{fig:throughput_generalization} demonstrates that ALOE shows stronger throughput, unseen-object generalization, and robustness to disturbances. We also perform ablation study on table~\ref{tab:critic_ablation} to deeper understand where ALOE's gains come from.
\vspace{-1mm}
\begin{table*}[t!]
\centering
\caption{\textbf{Where do ALOE's gains come from?} We isolate the critic design by removing action-level evaluation and off-policy TD bootstrapping. ALOE (full) uses TD with $Q^\pi(s, a)$; ALOE w.o action
level uses TD with all action masked during critic estimation; and ALOE w.o action level, w.o off-policy further replaces TD bootstrapping with MC returns $V^{\pi_D}(s)$ as RECAP. Under the same training recipe, we observe that full ALOE performs best across all real-world tasks.}
\label{tab:critic_ablation}
\footnotesize
\setlength{\tabcolsep}{4pt}
\renewcommand{\arraystretch}{0.95}
\begin{tabular}{lcccc}
\toprule
Variant (SR\%) & Phone & Laundry & Sorting & Overall \\
\midrule
ALOE (full) & \textbf{80.0 $\pm$ 5.2} & \textbf{58.7 $\pm$ 6.4} & \textbf{87.0 $\pm$ 4.3} & \textbf{75.2 $\pm$ 3.1} \\
w/o action-level & 65.0 $\pm$ 6.2 & 50.3 $\pm$ 6.5 & 80.0 $\pm$ 5.2 & 65.1 $\pm$ 3.4 \\
w/o action-level, w/o off-policy & 50.0 $\pm$ 6.5 & 45.0 $\pm$ 6.4 & 80.0 $\pm$ 5.2 & 58.3 $\pm$ 3.5 \\
\bottomrule
\end{tabular}
\end{table*}
\vspace{-1mm}
\subsection{Why Is the Learned Critic Reliable?}
Fig.~\ref{fig:critic_components} validates the two critic ingredients. Larger Q-chunks improve Phone-task throughput, and pessimistic aggregation outperforms mean aggregation on unseen objects, suggesting conservative scoring helps under distribution shift.
Fig.~\ref{fig:q_value_all} further evaluates action ranking under matched states, OOD calibration under meaningless actions, and the temporal localization of advantage along a trajectory.
Fig.~\ref{fig:adv_compare} shows that ALOE assigns high advantage to task-relevant approach-and-grasp actions, whereas AWR/RECAP can misattribute advantage to meaningless empty-grasp states.
Together, these diagnostics suggest that ALOE provides a more faithful action-level advantage estimate, rather than merely predicting coarse task progress.
% By contrast, the RECAP-style state-value baseline lacks explicit action conditioning, so its derived advantage can place large scores on meaningless empty-grasp actions.

% Table~\ref{tab:current_policy_corr} shows that ALOE has stronger correlation with realized returns on current-policy rollouts than the state-value AWR/RECAP baseline.

%% file: sec/conclusion.tex
\section{Conclusion}
We studied real-world post-training of large flow-based VLA policies under policy-mixed human-in-the-loop replay. ALOE learns action-level off-policy critics that evaluate the current policy's candidate actions. Across four real-world tasks, ALOE achieves the best final success rates and continues to improve over online RL iterations. Ablations and critic diagnostics show that these gains primarily come from action-level off-policy critic design.
\section{Limitations}
Even though ALOE consistently outperforms state-value baselines with the largest gains on precision-critical tasks. A current limitation is still the reliance on human supervision during real-world rollout. Integrating ALOE with automated reset pipelines is next step toward fully autonomous real-world VLA training.

%% file: sec/appendix.tex
%%%%%%%%%%%%%%%%%%%%%%%%%%%%%%%%%%%%%%%%%%%%%%%%%%%%%%%%%%%%
\section{Appendix}
This appendix provides reproducibility details and additional analysis, including the robotic platform, task specifications, reward definitions, human-in-the-loop data collection, training hyperparameters, evaluation protocols, supplementary quantitative results, and a formal interpretation of ALOE's policy improvement objective.
% ---------------------------------------------------------
\subsection{Robotic Platform}
\label{app:platform}
We describe the robotic platform and perception setup used for all tasks. Our experimental platform is built on two AgiBot robotic systems, which both consist of dual 7-DoF manipulator arms mounted on a mobile base with an adjustable waist mechanism. The platforms support modular end-effectors, enabling flexible configuration with either standard grippers or 6-DoF dexterous hands depending on the task requirements. The perception system includes multiple camera modalities: an RGB camera providing front-view coverage, along with RGB cameras mounted on each end-effector for close-up manipulation views. This multi-camera setup enables comprehensive visual observation from both global and local perspectives, which is essential for fine-grained manipulation tasks. For human-in-the-loop data collection, we employ VR-based teleoperation control. The operator uses a VR headset and controllers to command the robot. The VR controller hand gestures are mapped to end-effector translations and rotations, which are then converted to joint angle commands through inverse kinematics. The controller thumbsticks and buttons enable base and body movement control, while trigger buttons actuate the end-effectors. This teleoperation interface allows operators to provide real-time demonstrations, resets, and interventions during policy rollouts, enabling the collection of high-quality demonstration data and safety interventions when needed.

% ---------------------------------------------------------
\subsection{Post-training Process under Human-in-the-Loop Data Collection}
\label{app:hitl}

We describe our post-training process under a human-in-the-loop setting. For each task, we first collect several human demonstrations and use them to warm-start $\pi_{0.5}$ via imitation learning. At each RL iteration, we roll out $\pi_{0.5}$ while a human monitors execution. The human intervenes via teleoperation when unsafe behaviors are imminent, such as dropping or damaging the phone, which triggers early termination. If the policy fails or times out, the human resets the environment to the failure state and demonstrates corrective actions to complete the task. All trajectories, including both autonomous executions and human interventions, are stored in the replay buffer. The resulting dataset contains both successful and failed fragments, reflecting the heterogeneous nature of human-in-the-loop data collection. Data collection continues until a fixed number of successful trajectories is reached (e.g., 50 for the phone task). We then apply ALOE (Algorithm~\ref{alg:aloe_appendix}) to update the critic and actor before proceeding to the next RL iteration. Human intervention is frequent in early iterations but decreases as the policy improves.

\begin{algorithm}[h]
\caption{ALOE: Action-Level Off-Policy Evaluation Framework for VLA Policies}
\label{alg:aloe_appendix}
\textbf{Require:} initial VLA policy $\pi_\theta$; critic ensemble $\{Q^\pi_{\phi_i}\}_{i=1}^K$
with targets $\{Q^\pi_{\phi_i^-}\}_{i=1}^K$; replay buffer $\mathcal{D}$;
discount $\gamma$; hyperparameters $(\beta,\epsilon_{\mathrm{clip}},\omega)$; update steps $(N_Q,N_\pi)$.
\begin{algorithmic}[1]
\State Train $\pi_\theta$ by imitation learning on small demonstrations.
\For{$k = 1$ to $T$}
    \For{$i_{\mathrm{ep}} = 1$ to $N_{\mathrm{succ}}$}
        \State Deploy $\pi_\theta$ to autonomously execute the task under instruction $\ell$.
        \If{unsafe behavior or task failure is detected}
            \State Terminate early and label the segment as failure.
            \State Human teleoperates to recover the task.
        \EndIf
        \State Store all transitions (autonomous and human-corrected) into $\mathcal{D}$.
    \EndFor

    \For{$j = 1$ to $N_Q$}
        \State Sample minibatch $\mathcal{B}\subset\mathcal{D}$ and corresponding next actions from $\pi_\theta$.
        \State Compute TD target $y$ using Eq.~\eqref{eq:q_chunk_target}.
        \State Update critic ensemble by minimizing Eq.~\eqref{eq:critic_loss}.
        \State Update target networks $\phi_i^- \leftarrow \omega \phi_i + (1-\omega)\phi_i^-$.
    \EndFor

    \For{$j = 1$ to $N_\pi$}
        \State Sample minibatch $\mathcal{B}\subset\mathcal{D}$, and compute corresponding
        $Q^\pi_{\mathrm{pess}}(s_t, \mathbf{a}_{t:t+h}, \ell)$.
        \State Estimate $V(s, \ell)$ by sampling $a\sim\pi_\theta(\cdot \mid s, \ell)$
        and compute advantage.
        \State Compute weight $w$ using Eq.~\eqref{eq:aw_weight}.
        \State Update policy $\pi_\theta$ by minimizing Eq.~\eqref{eq:flow_actor}.
    \EndFor
\EndFor

\State \textbf{Return:} trained policy $\pi_\theta$.
\end{algorithmic}
\end{algorithm}

% ---------------------------------------------------------
\subsection{Training Details and Hyperparameters}
\label{app:training}
We summarize the actor architecture, critic design, optimization settings, and baseline implementations.

% ------------------------------------------------
\subsubsection{ALOE Actor}
We use the pretrained $\pi_{0.5}$ flow-matching VLA model~\cite{pi05} as the policy backbone, containing approximately 3B parameters. Policies are fine-tuned end-to-end with full parameters but vision encoders. At deployment, inference runs on a single RTX 4090 GPU. The control frequency is 30\,Hz, and manipulators operate under asynchronous control to reduce latency.

% ------------------------------------------------
\subsubsection{ALOE Critic}
We employ a Transformer-based multimodal critic to estimate chunked action-values. Visual observations are encoded using a pretrained SigLIP encoder, which remains frozen for stability. Proprioception and action chunks are projected into a shared embedding space of dimension $D=256$. An ensemble of $K$ Q-tokens is prepended to the input sequence as readout tokens, each mapped to a scalar Q-value via an independent MLP head. The input sequence is
\[
\mathbf{X} = [\mathbf{q}_1,\ldots,\mathbf{q}_K,\mathbf{s},\mathbf{a}_1,\ldots,\mathbf{a}_h,\mathbf{x}_1,\ldots,\mathbf{x}_V],
\]
where $h=50$ and $V=3$, giving sequence length $L=56$. The Transformer uses 6 layers and 8 attention heads. The critic is trained with Q-chunking TD targets (Eq.~\ref{eq:q_chunk_target}) and pessimistic aggregation. Unless otherwise specified, we use an ensemble size of $K=2$; in our component analysis, increasing the ensemble size to $K=4$ further improves OOD success rate under pessimistic aggregation. Target networks are updated via Polyak averaging with coefficient 0.005. Optimization uses AdamW with $\beta_1=0.9$, $\beta_2=0.95$, cosine learning-rate schedule (peak $3\times10^{-5}$, 100 warmup steps), gradient clipping at 1.0, discount factor $\gamma=0.99$, and batch size 256.

% ------------------------------------------------
\subsubsection{Baselines}
All baselines use the same $\pi_{0.5}$ backbone for fair comparison. For sorting and phone packing, training starts from an open-source checkpoint. For laundry folding, we warm-start with additional task data due to task difficulty.

\textbf{Behavior Cloning (BC).}
We follow~\cite{pi05} and minimize the flow-matching objective:
\begin{align}
\mathcal{L}_{\mathrm{BC}}
&=\mathbb{E}_{\mathcal{D}}
\big\|\boldsymbol{\epsilon}-\mathbf{a}_{t:t+h}
- f_\theta(\tilde{\mathbf{a}}_{t:t+h}, s_t, \ell)\big\|_2^2.
\end{align}

\textbf{DAgger.}
We implement human-in-the-loop DAgger~\cite{hgdagger,menda2019ensembledagger,ross2011dagger}. Each iteration aggregates new demonstrations and trains with the same objective as BC.

\textbf{AWR.}
We follow AWR~\cite{awr} and implement a distributional value critic~\cite{bellemare2017distributional} with $N=256$ bins. The value distribution is parameterized as
\[
p_\theta(V|s)=(p_0,\ldots,p_N),
\]
and the scalar value estimate is computed by taking the expectation over the discretized support:
\[
V(s)=\sum_i z_i \frac{e^{p_i}}{\sum_j e^{p_j}}.
\]
The targets are Monte Carlo returns discretized between $V_{\min}$ and $V_{\max}$. The critic shares the same architecture as the ALOE critic, except that it uses a single head ($K=1$). We use $V(s)$ to compute the advantage and perform AWR~\cite{awr} to extract the policy from the learned value function. This baseline provides a direct comparison with ALOE in terms of critic design.

\textbf{RECAP.}
We use the same critic design as in AWR. For the actor, however, we extract the policy using classifier-free guidance (CFG), following the original RECAP algorithm in $\pi_{0.6}$~\cite{recap}. Specifically, we label human-intervention samples, as well as non-intervention samples whose advantages rank in the top $10\%$, as positive samples, and treat all remaining samples as negative samples. We then append an \texttt{Advantage:true} or \texttt{Advantage:false} suffix to the prompt accordingly. During training, the action expert is trained to predict actions both with and without this condition, where the condition is dropped with probability $30\%$. At inference time, we set the label to \texttt{Advantage:true} and apply CFG with guidance scale $\beta=1$. This baseline provides a direct comparison between ALOE and RECAP.

\subsubsection{Optimization Settings}
Both ALOE and AWR critics/actors are continually fine-tuned from the previous iteration, following standard continual policy refinement practice in human-in-the-loop learning~\cite{wu2025robocopilothumanintheloopinteractiveimitation}. Policies are evaluated using their respective critics. We set the advantage-weighting temperature to $\beta=0.5$ for ALOE actor extraction. For sparse-reward training, the failure penalty $C_{\mathrm{fail}}$ is set to the average trajectory length of each task, so failures receive a penalty comparable to terminating after a typical unsuccessful rollout. In each online RL iteration, we perform $N_Q$ critic gradient steps and $N_\pi$ actor gradient steps; unless otherwise specified, we set $N_Q=10000$ and $N_\pi=5000$ across tasks.
\subsubsection{Data Collection Details}
During data collection, we train a new policy and proceed to the next iteration once each round accumulates 50 successful trajectories.
For the more challenging Phone Assembly tasks, we instead collect around 100 successful trajectories per round before retraining.
During rollout, if the policy makes a recoverable mistake, a human operator takes over from the failure point and completes the trajectory.
If the mistake is unrecoverable, the attempt is terminated and restarted.
We further visualize this data collection process on the Phone Assembly task.
Fig.~\ref{fig:data_pie} shows the mixture data distribution accumulated across iterations, including frames generated by different policies, and human intervention frames.
As training progresses, the proportion of human intervention frames decreases, while policy-generated frames occupy a larger share of the collected data.
Fig.~\ref{fig:data_curve} summarizes this trend more directly: the intervention frame ratio steadily drops across iterations, indicating that the policy requires less human takeover as iterative training proceeds.

\begin{figure}
    \centering
    \includegraphics[width=\linewidth]{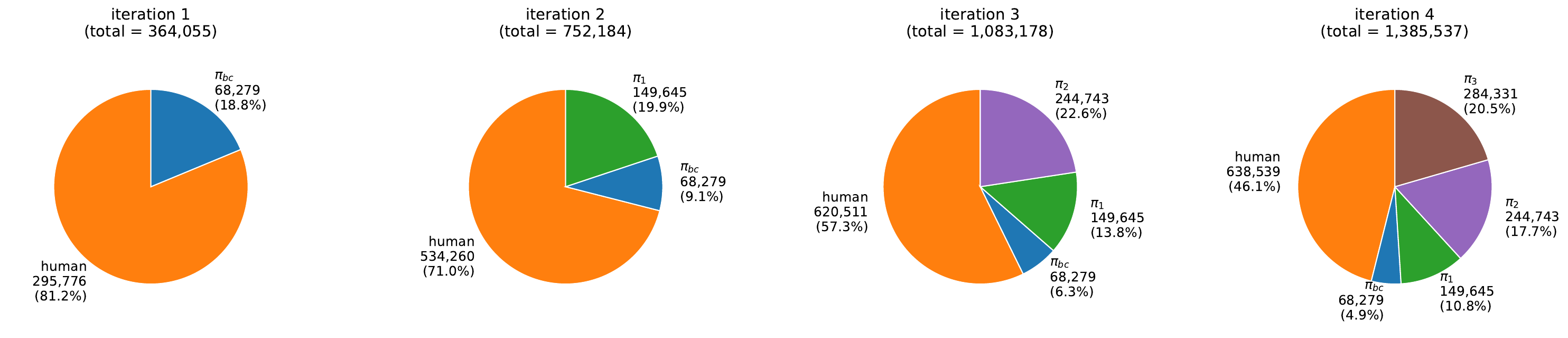}
    \caption{
        \textbf{Composition of accumulated frames across policy improvement iterations on Phone Assembly task.}
        As training progresses, the fraction of human intervention frames decreases, while policy-generated frames from successful and failed attempts increase, reflecting a gradual shift from human-assisted recovery toward autonomous data collection.
    }
    \label{fig:data_pie}
\end{figure}
\begin{figure}
    \centering
    \includegraphics[width=0.5\linewidth]{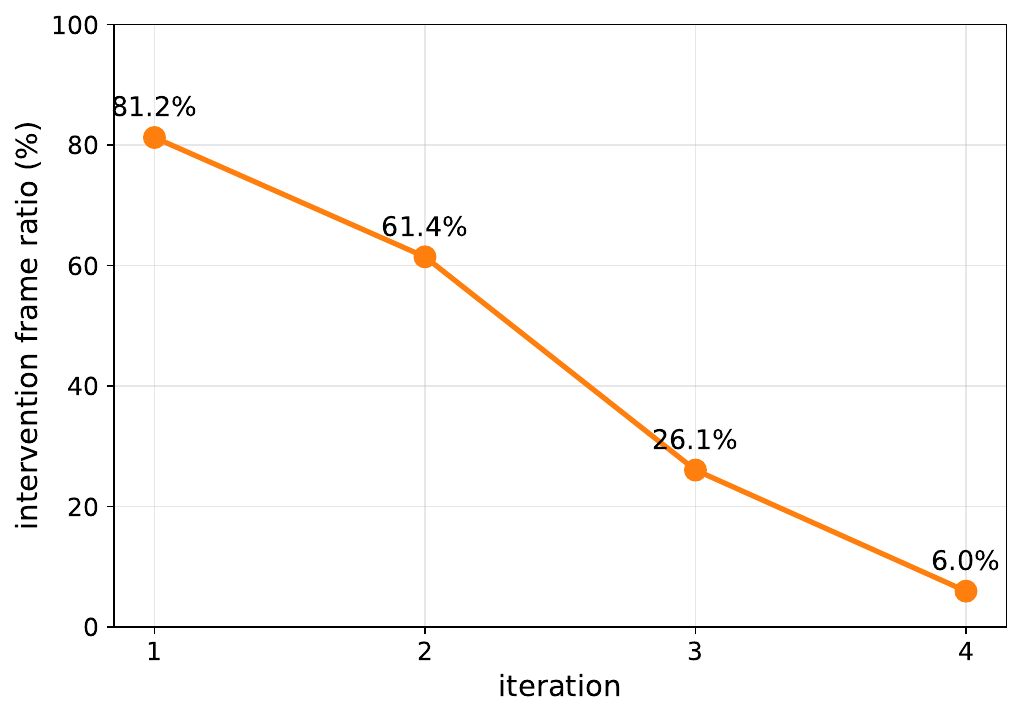}
    \caption{
        \textbf{Human intervention frame ratio across policy improvement iterations on Phone Assembly task.}
        The intervention ratio steadily decreases from 81.2\% to 6.0\%, indicating that the learned policy requires substantially less human takeover as iterative training proceeds.
    }
    \label{fig:data_curve}
\end{figure}

\subsection{Tasks and Environments}
\label{app:tasks}
\subsubsection{Pack Smart Phone}

Pack Smart Phone is a bimanual fine-manipulation task: pick the phone from the table, place it fully inside a rigid container (17.5cm $\times$ 8.6cm), position the container, and close the lid. The phone case (17.8cm $\times$ 8.8cm) has minimal clearance, making insertion and lid closure precision-critical. Success requires continuous pose alignment and recovery from positional variability during placement. We use joint angles and gripper pose as the action output.

\begin{table}[h]
\centering
\caption{Task parameters for Pack Smart Phone task.}
\label{tab:pack_phone_params}
\begin{tabular}{lp{0.5\columnwidth}}
\toprule
\textbf{Parameter} & \textbf{Value} \\
\midrule
Action space & 14-dimensional \\
Observation space & RGB from left/right wrist camera, RGB from head camera, 32 proprio \\
Prompt & Place the phone in the container, then cover it with the lid \\
Initial offline demonstrations & 50 \\
Max episode length & 2000 \\
Action chunk length & 50 \\
Execution chunk length & 25 \\
Reset method & Human reset \\
Randomization range & 3 cm in x and y \\
\bottomrule
\end{tabular}
\end{table}

The reward function is defined as follows:
\begin{itemize}
    \item \textbf{Success}: If the phone is completely placed in the container and the lid is successfully closed, the reward is 0.
    \item \textbf{Intermediate steps}: The intermediate process receives a reward of -1.
    \item \textbf{Failure cases}: The reward is $-C_{\text{fail}}$ if the phone accidentally drops, fails to accurately insert into the container, causes damage during insertion, or if the time step exceeds the max episode length.
\end{itemize}

\begin{figure}
    \begin{subfigure}[b]{0.3\linewidth}
    \centering
    \includegraphics[width=\linewidth]{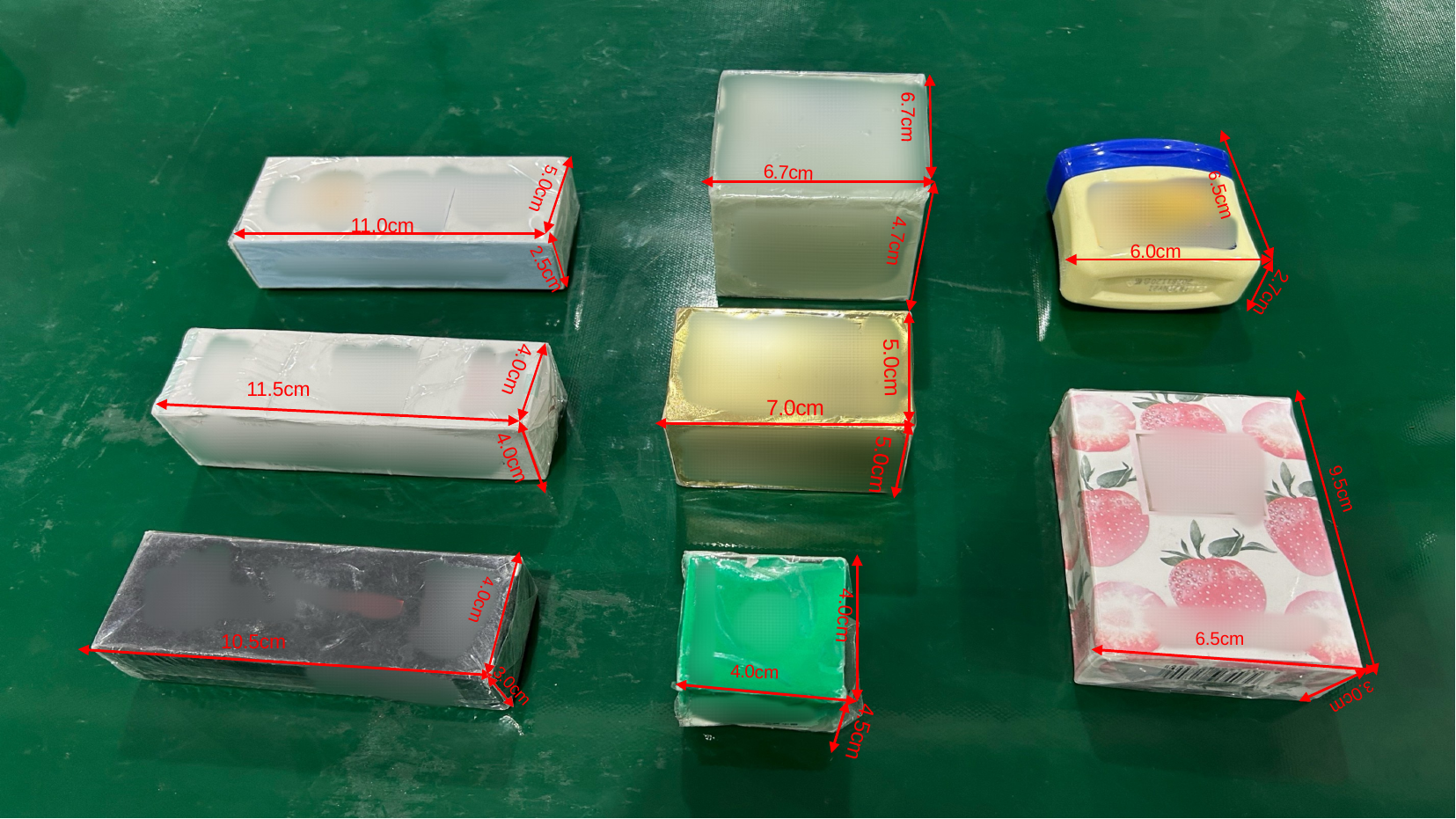}
    \caption{Unseen objects.}
    \label{fig:unseen_objects}
    \end{subfigure}
    \begin{subfigure}[b]{0.3\linewidth}
    \centering
    \includegraphics[width=\linewidth]{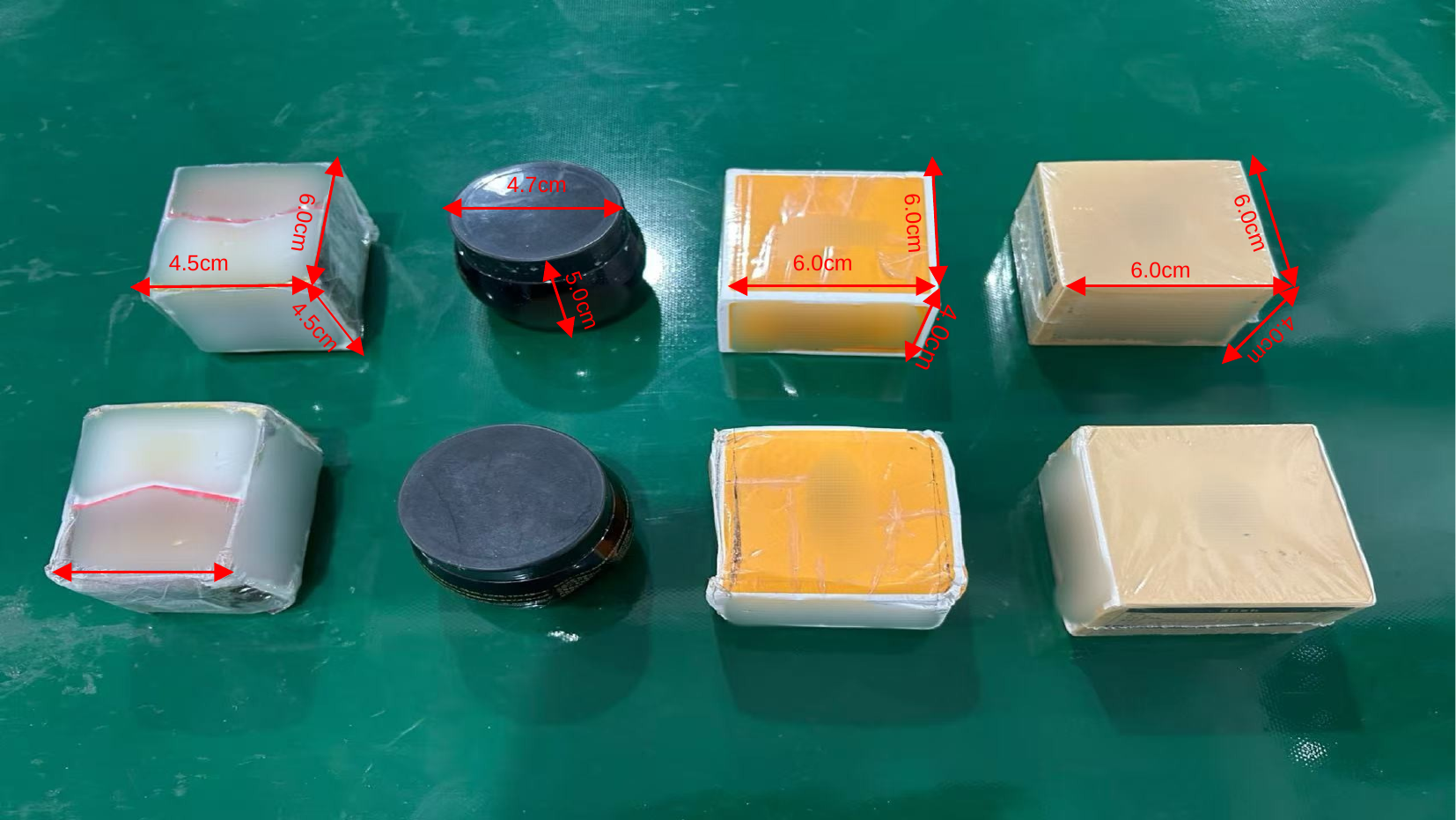}
    \caption{Seen objects.}
    \label{fig:seen_objects}
    \end{subfigure}
    \begin{subfigure}[b]{0.37\linewidth}
    \centering
    \includegraphics[width=\linewidth]{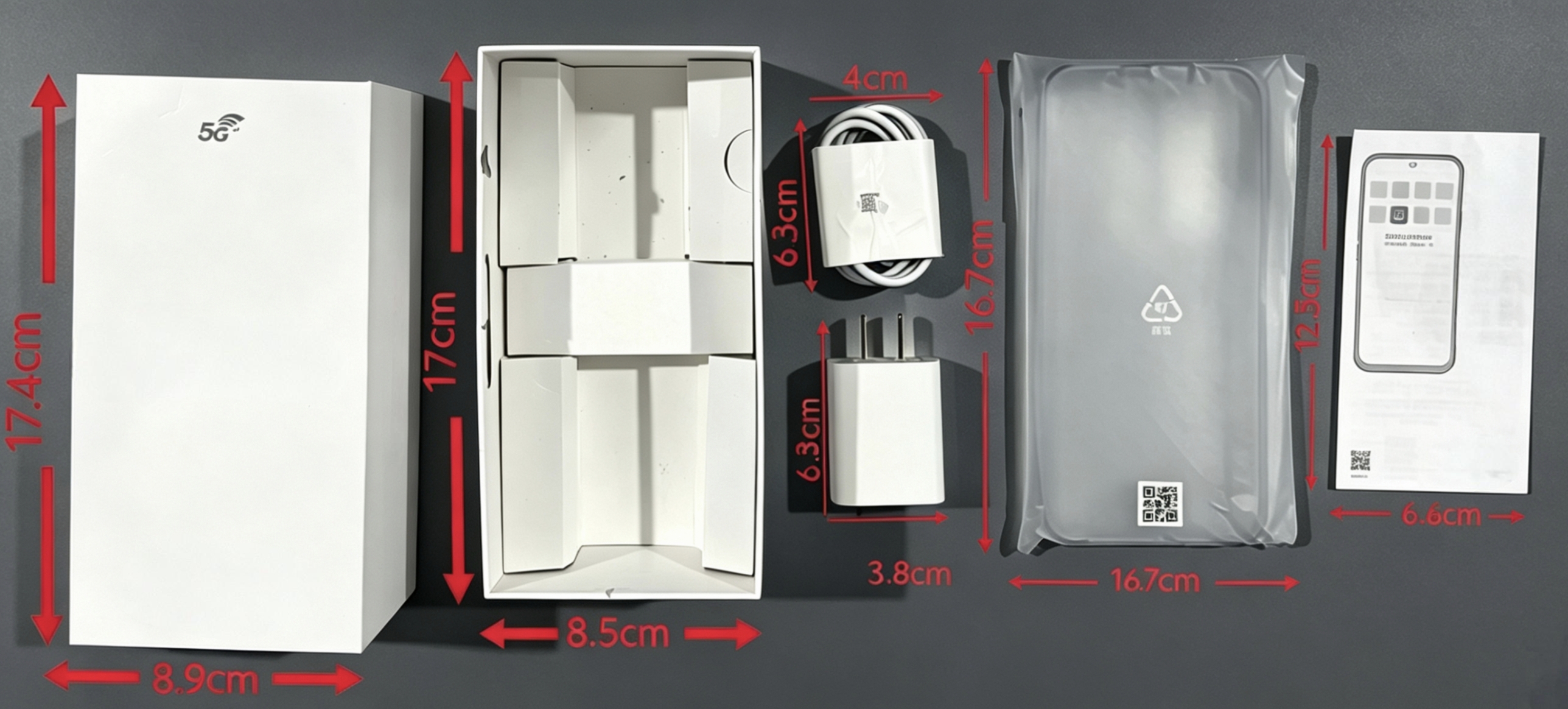}
    \caption{Phone accessories.}
    \label{fig:phone_objects}
    \end{subfigure}
    \caption{Objects used in our manipulation tasks.}

\end{figure}

\subsubsection{Folding Laundry}

Folding Laundry is a bimanual long-horizon task with deformable objects. The robot retrieves a garment from a rigid basket (35.0cm $\times$ 24.0cm), flattens it via coordinated bimanual grasping and spreading, then executes sequential folds to a target configuration. Garments vary in color (six types) and size, demanding robust visual perception and generalization of folding strategies. The high-dimensional configuration space and long episode duration necessitate recovery from entanglement and irregular states.

\begin{table}[ht]
\centering
\caption{Task parameters for Folding Laundry task.}
\label{tab:folding_laundry_params}
\begin{tabular}{lp{0.5\columnwidth}}
\toprule
\textbf{Parameter} & \textbf{Value} \\
\midrule
Action space & 14-dimensional \\
Observation space & RGB from left/right wrist camera, RGB from head camera, 32 proprio \\
Prompt & Fold the t-shirt \\
Initial offline demonstrations & 150 \\
Max episode length & 10000 \\
Action chunk length & 50 \\
Execution chunk length & 5 \\
Reset method & Human reset \\
Randomization range & Randomly place clothes in the box \\
\bottomrule
\end{tabular}
\end{table}

The reward function is defined as follows:
\begin{itemize}
    \item \textbf{Success}: If the clothes are neatly folded, the reward is 0.
    \item \textbf{Intermediate steps}: The intermediate process receives a reward of -1.
    \item \textbf{Failure cases}: The reward is $-C_{\text{fail}}$ if the gripper performs an empty grasp or grasps multiple layers of clothes, if the robot repeatedly shakes the clothes, if the robot begins folding without flattening the garment first, if the robot collides with the table, if the robot re-flattens the garment during the folding process, or if the time step exceeds the max episode length.
\end{itemize}

\subsubsection{Product Sorting}

Product Sorting is a bimanual pick-place task. The robot empties two material bins (left and right) by picking objects and placing them onto a front conveyor belt. Each arm is constrained to its own bin (left gripper from left bin, right from right); the task is complete when both bins are empty.

\begin{table}[h]
\centering
\caption{Task parameters for Product Sorting task.}
\label{tab:product_sorting}
\begin{tabular}{lp{0.5\columnwidth}}
\toprule
\textbf{Parameter} & \textbf{Value} \\
\midrule
Action space & 14-dimensional \\
Observation space & RGB from left/right wrist camera, RGB from head camera, 32 proprio \\
Prompt & Retrieve the objects from the front bin and place them onto the conveyor belt. \\
Initial offline demonstrations & 50 \\
Max episode length & 10000 \\
Action chunk length & 50 \\
Execution chunk length & 5 \\
Reset method & Human reset \\
Randomization range & Randomly place objects in the bins \\
\bottomrule
\end{tabular}
\end{table}

The reward function is defined as follows:
\begin{itemize}
    \item \textbf{Success}: If both material bins are successfully emptied, the reward is 0.
    \item \textbf{Intermediate steps}: The intermediate process receives a reward of -1.
    \item \textbf{Failure cases}: The reward is $-C_{\text{fail}}$ if any of the following situations occur: gripper positioning failure, empty gripper grasp, the gripper squeezing other objects during the grasping process, the robotic arm colliding with the material bin during movement, the gripper releasing the material in advance while the material is still in the air, collision occurring during the gripper placement process, or the time step exceeding the maximum episode length.
\end{itemize}

\subsubsection{Phone Assembly}

Phone Assembly is a bimanual, long-horizon manipulation task. The robot picks phone accessories, including a phone case, manual, charging cable, and power plug, from material trays and places them at designated locations. The task is complete when all accessories are correctly placed. The full operation takes around 1 min 20 s, and placing the power plug requires high manipulation precision, making this a challenging task scenario.

\begin{table}[h]
\centering
\caption{Task parameters for Phone Assembly task.}
\label{tab:phone_accessories_params}
\begin{tabular}{lp{0.5\columnwidth}}
\toprule
\textbf{Parameter} & \textbf{Value} \\
\midrule
Action space & 14-dimensional \\
Observation space & RGB from left/right wrist camera, RGB from head camera, 14 proprio. \\
Prompt & Place the cover and manual on the right side of the phone box. Then, place the charger plug and the cable inside the box. \\
Initial offline demonstrations & 150 \\
Max episode length & 10000 \\
Action chunk length & 50 \\
Execution chunk length & 5 \\
Randomization range & 1 cm in x and y \\
\bottomrule
\end{tabular}
\end{table}

The reward function is defined as follows:
\begin{itemize}
    \item \textbf{Success}: If four accessories are correctly placed, the reward is 0.
    \item \textbf{Intermediate steps}: The intermediate process receives a reward of -1.
    \item \textbf{Failure cases}: The reward is $-C_{\text{fail}}$ if any of the following situations occur: empty gripper grasp, object dropping during grasping or transfer, deviation between the final placement and the target location, or the time step exceeds the max episode length. Typical placement failures include the power adapter or charging cable not being fully placed into the designated slot, or the phone case being significantly misaligned with the phone box.
\end{itemize}

\subsection{Evaluation Protocol}
\label{app:eval_protocol}

We define success rate, throughput, unseen object evaluation, and robustness protocol.

\textbf{Success rate.} Success rate is the fraction of successful episodes over the total number of evaluation episodes across multiple evaluation rounds. An episode is considered successful when the task is completed according to the reward-defined success criteria; it is considered completed (for throughput) when it terminates due to success, failure, or timeout.

\textbf{Throughput.} Throughput is defined as completed episodes per hour and captures both execution speed and reliability under a fixed time budget. Throughput is measured as successful task completions per hour under successful executions; reset and scene setup time are excluded. This exclusion is applied consistently across all compared methods, so relative throughput comparisons remain valid (Fig.~\ref{fig:throughput_generalization}, left).

\textbf{Unseen object zero-shot generalization.} To evaluate generalization to novel objects, we replace the objects in the Product Sorting task with objects that were never seen during training (different shapes and colors). See Fig.~\ref{fig:seen_objects} and Fig.~\ref{fig:unseen_objects} for examples of training vs.\ unseen objects. The policy is evaluated without fine-tuning; this setting assesses the VLA's ability to generalize to unseen object appearances and geometries (Fig.~\ref{fig:throughput_generalization}, middle).

\textbf{Robustness evaluation.} During evaluation, we inject random perturbations to the garment being manipulated (e.g., in Folding Laundry) while the policy is executing. We then measure whether the VLA can re-adjust its actions and still complete the task successfully. This protocol evaluates recovery from unexpected disturbances (Fig.~\ref{fig:throughput_generalization}, right).

\subsection{Additional Results}
\label{app:additional_results}

Table~\ref{tab:bc_results} reports the performance of behavior cloning (BC) on the four tasks. All baselines start from the BC checkpoint before iterative RL.
\begin{table}[h]
\centering
\caption{BC performance on four tasks.}
\label{tab:bc_results}
\begin{tabular}{lcccc}
\toprule
 & Pack Phone & Folding Laundry & Product Sorting & Phone Assembly \\
\midrule
BC & 5.0 & 15.0 & 5.0 & 10.0 \\
\bottomrule
\end{tabular}
\end{table}

We also analyze the effect of actor updates by keeping the same action-level critic while changing only the policy extraction rule on table~\ref{tab:actor_side_ablation}. We find that success rate (SR) differences are modest, but ALOE extraction yields lower actor-loss variance and more stable optimization than the AWR-style alternative. \textbf{This suggests that our actor-side contribution mainly improves stability, while larger performance gains come from critic-side design.}
 We include an additional critic visualization to complement the quantitative diagnostics in Table~\ref{tab:current_policy_corr}, which shows that ALOE has stronger correlation with realized returns on current-policy rollouts than the state-value AWR/RECAP baseline.
\begin{table}[h]
\centering
\small
\caption{Actor-side ablation results.}
\label{tab:actor_side_ablation}

\begin{tabular}{lcccc}
\toprule
Actor-side ablation (SR\%) 
& Phone 
& Laundry 
& Sorting 
& Actor Loss std ($\downarrow$) \\
\midrule

ALOE policy extraction            
& 80.0 
& 58.7 
& 87.0 
& \textbf{0.004} \\

ALOE policy extraction w.o.\ clip  
& 75.0 
& 55.5 
& 85.0 
& 0.032 \\

\bottomrule
\end{tabular}
\end{table}

\begin{table*}[h]
\centering
\footnotesize
\setlength{\tabcolsep}{4pt}
\renewcommand{\arraystretch}{0.95}
\caption{\textbf{Critic-return correlation on current-policy rollouts.} We evaluate critic-return correlation on online rollouts from current policy. ALOE achieves higher correlation than AWR/RECAP during
VLA iterative online rollouts, indicating ALOE critic tracks current-policy action quality more reliably.}
\label{tab:current_policy_corr}
\begin{tabular}{lcc}
\toprule
$\pi_{\mathrm{curr}}$ Rollouts & Spearman ($\uparrow$) & Pearson ($\uparrow$) \\
\midrule
ALOE & \textbf{0.9283} & \textbf{0.9580} \\
AWR/RECAP & 0.9090 & 0.7640 \\
\bottomrule
\end{tabular}
\end{table*}

Fig.~\ref{fig:q_visualization_additional} provides another visualization of the learned $Q(s,a)$ along a trajectory.
\begin{figure}[t]
    \centering
    \includegraphics[width=\linewidth]{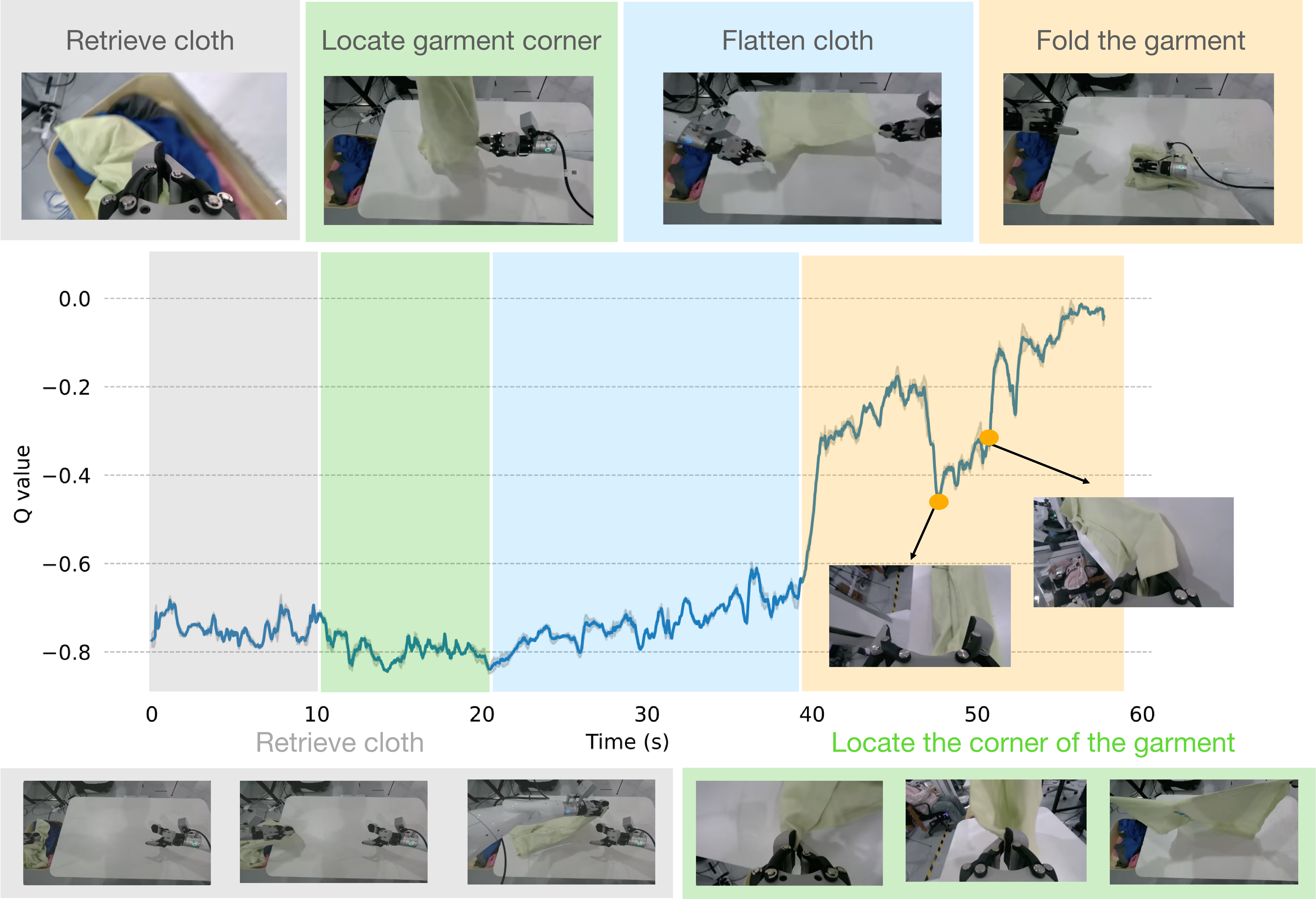}
    \caption{Additional $Q(s,a)$ visualization along a representative trajectory.}
    \label{fig:q_visualization_additional}
\end{figure}
\begin{table*}[t]
\centering
\footnotesize
\setlength{\tabcolsep}{4pt}
\renewcommand{\arraystretch}{0.95}
\caption{Final success rates on tasks with common baselines, corresponding to Fig.~\ref{fig:all_results}.}
\label{tab:final_success_common_appendix}
\begin{tabular}{lccc}
\toprule
Method & Pack Phone & Folding Laundry & Product Sorting \\
\midrule
DAgger & 43.3 $\pm$ 6.4 & 43.0 $\pm$ 1.0 & 82.0 $\pm$ 5.0 \\
AWR & 65.0 $\pm$ 6.2 & 52.2 $\pm$ 2.0 & 80.0 $\pm$ 5.2 \\
ALOE & \textbf{80.0 $\pm$ 5.1} & \textbf{58.7 $\pm$ 0.8} & \textbf{87.0 $\pm$ 4.3} \\
\bottomrule
\end{tabular}
\end{table*}

\begin{table}[t]
\centering
\footnotesize
\setlength{\tabcolsep}{5pt}
\renewcommand{\arraystretch}{0.95}
\caption{Final success rates on Phone Assembly, corresponding to Fig.~\ref{fig:all_results}.}
\label{tab:final_success_assembly_appendix}
\begin{tabular}{lc}
\toprule
Method & Phone Assembly SR (\%) \\
\midrule
DAgger & 55.0 $\pm$ 6.4 \\
AWR & 80.0 $\pm$ 5.2 \\
RECAP & 70.0 $\pm$ 5.9 \\
ALOE & \textbf{95.0 $\pm$ 2.8} \\
\bottomrule
\end{tabular}
\end{table}
\begin{table*}[t]
\centering
\footnotesize
\setlength{\tabcolsep}{4pt}
\renewcommand{\arraystretch}{0.95}
\caption{Policy improvement across online RL iterations, corresponding to Fig.~\ref{fig:policy_improvement}. Iteration 0 denotes the BC warm-up policy.}
\label{tab:policy_improvement_appendix}
\begin{tabular}{llp{0.58\linewidth}}
\toprule
Task & Method & Success rate over iterations (\%) \\
\midrule
Phone Packing & ALOE & 5.0, 10.0, 15.0, 45.0, 65.0, 80.0 \\
Phone Packing & DAgger & 5.0, 5.0, 5.0, 25.0, 55.0, 43.3 \\
Phone Packing & AWR & 5.0, 5.0, 5.0, 10.0, 40.0, 65.0 \\
\midrule
Product Sorting & ALOE & 5.0, 20.0, 40.0, 90.0 \\
Product Sorting & DAgger & 5.0, 5.0, 15.0, 80.0 \\
Product Sorting & AWR & 5.0, 5.0, 38.0, 75.0 \\
\midrule
Phone Assembly & ALOE & 10.0, 65.0, 75.0, 95.0 \\
Phone Assembly & RECAP & 10.0, 55.0, 50.0, 70.0 \\
Phone Assembly & AWR & 10.0, 60.0, 55.0, 80.0 \\
Phone Assembly & DAgger & 10.0, 20.0, 40.0, 55.0 \\
\bottomrule
\end{tabular}
\end{table*}
\begin{table*}[t]
\centering
\footnotesize
\setlength{\tabcolsep}{4pt}
\renewcommand{\arraystretch}{0.95}
\caption{Efficiency, zero-shot generalization, and robustness results corresponding to Fig.~\ref{fig:throughput_generalization}. Throughput is measured in successful task completions per hour; success-rate error bars denote binomial standard errors.}
\label{tab:efficiency_generalization_appendix}
\begin{tabular}{lccc}
\toprule
Method & Phone Throughput (tasks/hour) & Unseen Sorting SR (\%) & Laundry Robustness SR (\%) \\
\midrule
DAgger & 86.1 & 62.5 $\pm$ 6.2 & 88.5 $\pm$ 4.1 \\
AWR & 82.9 & 38.8 $\pm$ 6.3 & 84.6 $\pm$ 4.7 \\
ALOE & \textbf{92.5} & \textbf{68.8 $\pm$ 6.0} & \textbf{92.3 $\pm$ 3.4} \\
\bottomrule
\end{tabular}
\end{table*}
\begin{table}[t]
\centering
\footnotesize
\setlength{\tabcolsep}{5pt}
\renewcommand{\arraystretch}{0.95}
\caption{Q-chunking analysis corresponding to Fig.~\ref{fig:critic_components}. Throughput is measured in tasks per hour; success-rate error bars denote binomial standard errors.}
\label{tab:q_chunking_appendix}
\begin{tabular}{lcc}
\toprule
Variant & Throughput (tasks/hour) & Success Rate (\%) \\
\midrule
1-step & 52.9 & 70.0 $\pm$ 5.9 \\
TD-5 & 72.6 & 70.0 $\pm$ 5.9 \\
TD-15 & 75.2 & 80.0 $\pm$ 5.2 \\
TD-50 & \textbf{83.3} & \textbf{85.0 $\pm$ 4.6} \\
\bottomrule
\end{tabular}
\end{table}
\begin{table*}[t]
\centering
\footnotesize
\setlength{\tabcolsep}{5pt}
\renewcommand{\arraystretch}{0.95}
\caption{Pessimistic ensemble analysis corresponding to Fig.~\ref{fig:critic_components}. Error bars denote binomial standard errors.}
\label{tab:pessimistic_ensemble_appendix}
\begin{tabular}{lc}
\toprule
Variant & OOD Success Rate (\%) \\
\midrule
Mean-2 & 45.0 $\pm$ 6.4 \\
Min-1 & 60.0 $\pm$ 6.3 \\
Min-2 & 70.0 $\pm$ 5.9 \\
Min-4 & \textbf{75.0 $\pm$ 5.6} \\
\bottomrule
\end{tabular}
\end{table*}

Toward vision encoder stability as a sanity check, we also train the SigLIP vision encoder during critic learning.
Unfreezing SigLIP accelerates TD-loss convergence in our runs (critic loss 0.0021 vs.\ 0.0079 after 1k steps) without showing an unstable trend.
Because it substantially increases compute and memory cost, we keep SigLIP frozen in the main experiments. We also provide the numerical values corresponding to the bar plots in the main paper. Unless otherwise noted, error bars for success rates denote binomial standard errors over 60 evaluation trials. Tables~\ref{tab:final_success_common_appendix} and~\ref{tab:final_success_assembly_appendix} report the final-policy success rates shown in Fig.~\ref{fig:all_results}.

Table~\ref{tab:policy_improvement_appendix} reports the iteration-wise success rates shown in Fig.~\ref{fig:policy_improvement}.

Table~\ref{tab:efficiency_generalization_appendix} reports the efficiency, zero-shot generalization, and robustness values shown in Fig.~\ref{fig:throughput_generalization}.

Tables~\ref{tab:q_chunking_appendix} and~\ref{tab:pessimistic_ensemble_appendix} report the critic component analysis values shown in Fig.~\ref{fig:critic_components}.
\clearpage
\subsection{Unclipped Policy Improvement Interpretation}
\label{app:proof_main}
The following statement formalizes the connection between the unclipped advantage-weighted actor objective and the KL-constrained policy improvement problem.
\begin{theorem}[Constraint Policy Improvement]
\label{thm:support_appendix}
For a given state $s$, let $\pi_{\mathrm{ref}}$ be the implicit behavior policy representing the data distribution in $\mathcal{D}$. For the unclipped advantage-weighted objective, the optimal solution matches the analytical solution to the following KL-constrained optimization problem:
\begin{equation}
\begin{aligned}
\max_{\pi} \quad 
& \mathbb{E}_{\mathbf{a}_{t:t+h} \sim \pi(\cdot \mid s_t, \ell)} \big[ Q(s_t, \mathbf{a}_{t:t+h}, \ell) \big] \\
\text{s.t.} \quad
& D_{\mathrm{KL}}\!\left(
\pi(\cdot \mid s_t, \ell)\,\|\,\pi_{\mathrm{ref}}(\cdot \mid s_t, \ell)
\right) \le \epsilon.
\end{aligned}
\end{equation}
The KL radius is implicitly controlled by the Lagrange multiplier $\beta$.
\end{theorem}

\begin{proof}
    Fix a state $s$ (and omit $\ell$ for brevity). Consider the KL-constrained policy improvement problem:
    \begin{equation}
    \begin{aligned}
    \max_{\pi}\ & \int \pi(a)\, Q(a)\, da \\
    \text{s.t.}\ & \int \pi(a)\,\log\frac{\pi(a)}{\pi_{\mathrm{ref}}(a)}\,da \le \epsilon,\\
    & \int \pi(a)\,da = 1.
    \end{aligned}
    \end{equation}
    
    Form the Lagrangian
    \begin{equation}
    \begin{aligned}
    \mathcal{L}(\pi,\beta,\alpha)
    = {} & \int \pi(a)\,Q(a)\,da \\
    & - \beta \!\left(\int \pi(a)\log\frac{\pi(a)}{\pi_{\mathrm{ref}}(a)}da-\epsilon\right) \\
    & + \alpha\!\left(\int \pi(a)da-1\right).
    \end{aligned}
    \end{equation}
    
    Taking the functional derivative and setting it to zero:
    \begin{equation}
    \frac{\delta \mathcal{L}}{\delta \pi(a)}
    = Q(a)
    - \beta\!\left(\log\frac{\pi(a)}{\pi_{\mathrm{ref}}(a)}+1\right)
    + \alpha
    = 0,
    \end{equation}
    which yields
    \begin{equation}
    \pi^*(a)
    \propto
    \pi_{\mathrm{ref}}(a)
    \exp\!\left(\frac{Q(a)}{\beta}\right).
    \end{equation}
    
    Now consider the advantage-weighted objective
    \[
    \mathbb{E}_{a\sim\pi_{\mathrm{ref}}}
    \big[
    \exp(A^{\pi}(a)/\beta)\log\pi_\theta(a)
    \big].
    \]
    Since
    \[
    A^{\pi}(a)=Q(a)-V^{\pi}(s),
    \]
    and $V^{\pi}(s)$ is constant w.r.t.\ $a$, the weighting is proportional to $\exp(Q(a)/\beta)$. Therefore, the maximum-likelihood solution satisfies
    \begin{equation}
    \pi_\theta(a)
    \propto
    \pi_{\mathrm{ref}}(a)
    \exp\!\left(\frac{Q(a)}{\beta}\right),
    \end{equation}
    which matches the solution to the KL-constrained problem. Hence the actor loss implements the optimal constraint policy improvement.
    \end{proof}